\documentclass{article}

\usepackage{spconf,amsmath,graphicx,amsfonts, comment, amsthm, subcaption}
\usepackage{multirow}
\usepackage{hyperref}
\usepackage{mwe}
\usepackage{algorithm}
\usepackage[noend]{algpseudocode}

\usepackage{ifthen}
\newboolean{WithProof}
\setboolean{WithProof}{true}
\newcommand{\ProofLocation}{
\ifthenelse{\boolean{WithProof}}{%if true
in Appendix
}{%if false
\href{https://claroche-r.github.io/PnP_LADMM/proof}{here}}
}
\newcommand{\prox}{\ensuremath{\operatorname{prox}}}

\usepackage{color}

\usepackage{soul}
\newcommand{\old}[1]{}

\newtheorem{theorem}{Theorem}
\newtheorem{lemma}[theorem]{Lemma}
\newtheorem{proposition}{Proposition}
\newtheorem{assumption}{Assumption}

\def\L{{\mathcal L}}
\def\LH{\|H\|^2} % L_{H^T H}
\def\denoiser{\mathcal{D}}
\def\sigmad{{\sigma_d}}

\title{Provably Convergent Plug \& Play Linearized ADMM, applied to Deblurring Spatially Varying Kernels}

\name{Charles Laroche$^{\star \dag}$ \qquad Andrés Almansa$^{\star}$ \qquad Eva Coupeté$^{\dag}$ \qquad Matias Tassano$^{\ddag}$}%\textsuperscript{0}$}

\address{$^{\star}$ CNRS \& Université Paris Cité, MAP5 \qquad
$^{\dag}$ GoPro \qquad
$^{\ddag}$ Meta Inc}

\begin{document}
\maketitle
\footnotetext[1]{Work mostly done while Matias Tassano was at GoPro France.} \footnotetext[2]{We thank Pauline Tan, Arthur Leclaire, and Samuel Hurault for fruitful discussions. This work has been partially
funded by the French National Research and Technology
Agency (ANRT) and GoPro France.}
\begin{abstract}
Plug \& Play methods combine proximal algorithms with denoiser priors to solve inverse problems. These methods rely on the computability of the proximal operator of the data fidelity term. In this paper, we propose a Plug \& Play framework based on linearized ADMM that allows us to bypass the computation of intractable proximal operators. We demonstrate the convergence of the algorithm and provide results on restoration tasks such as super-resolution and deblurring with non-uniform blur.
\end{abstract}
\section{Introduction}
\label{sec:intro}
%
% Introduction to inverse problems
% ----------------------------------------------------------
Many image restoration tasks can be formulated as inverse problems:
\begin{equation}\label{equ:inverse-problem}
    y = Hx + \varepsilon
\end{equation}
with $y \in \mathbb{R}^p$ the degraded image, $x \in \mathbb{R}^n$ the unknown clean image, $H \in \mathbb{R}^{p*n}$ the degradation matrix and $\varepsilon$ is the measurement noise.
Such tasks include denoising, deblurring, super-resolution, compressed sensing and so on. 
The reconstructed image $x$ can be obtained by maximizing the posterior $p(x|y) \propto p(y|x)p(x)$. Equivalently the posterior maximization or MAP estimator can be expressed as 
\begin{equation}
 \label{equ:map-optim}
%     x_{MAP} = \arg\min_x{h(Hx)+ \lambda f(x)}
     x_{MAP} = \arg\min_x\underbrace{h(Hx)+ \lambda f(x)}_{E(x)}
 \end{equation}
where $h(x) = -\log(p(y|x))$ is known as the data fitting term or negative log-likelihood and $\lambda f(x) = -\log(p(x))$ is the regularization term or negative log-prior. 
%
% Related work regularization
Classical approaches used convex regularization terms such as
Tikhonov~\cite[ch7]{milanfar2011SRbook},
Total~Variation~\cite{paramanand2013non} 
or wavelet-$\ell_1$~\cite{escande_sparse_2013} for example. More recently, \cite{venkatakrishnan2013} introduced \emph{Plug \& Play (PnP)} algorithms that enable the use of pretrained neural networks as implicit regularizers. PnP algorithms use a proximal splitting algorithm to solve the optimisation problem \eqref{equ:map-optim}, and then substitute the regularization subproblem by a pretrained denoiser.
The focus of this work is a variation of the alternating direction method of multipliers (ADMM) algorithm~\cite{prox_algo}, but the same idea has been extended to other splitting schemes including Primal Dual Splitting Splitting (PDS)~\cite{dpnp_dualprimal, Meinhardt2017}, fast iterative shrinkage~\cite{dpnp_fista}, and gradient descent~\cite{RED,Laumont2022pnpsgd}.

%
%
%% Intractable Prox_h(H . )
\old{Proximal-based }PnP algorithms like ADMM \old{or Primal-Dual} involve the computation at each iteration of the proximal operator of the data fitting term
\begin{align}
\label{equ:prox_g}
    \prox_{\alpha h(H \cdot)}(x) = \arg\min_z {\frac{1}{2\alpha}}\|x-z\|_2^2 + h(Hz).
\end{align}
This computation admits a fast closed form solution for many inverse problems like super-resolution~\cite{FastSR} or deconvolution~\cite{dpnp_theory_bounded}. For more complex tasks like deblurring with spatially-varying blur for example, the exact solution of \eqref{equ:prox_g} is computationally intractable, and even approximate solutions can be computationally expensive.
%\\
A common solution in such cases is to use ISTA \cite{Xu2020}, RED \cite{RED} or SGD \cite{Laumont2022pnpsgd} schemes where the more computationally friendly gradient $H^T \nabla h(H \cdot)$ is computed instead of the intractable proximal operator $\prox_{\alpha h(H \cdot)}$. Nevertheless this solution employing the gradient is not ideal because PnP-ADMM usually converges in far fewer iterations and is more robust to initial conditions than its gradient-based counterparts~\cite{Ahmad2020}. 

We propose in Section~\ref{sec:model} a linearized version of PnP-ADMM which preserves the benefits of PnP-ADMM while avoiding the costly proximal computation. This approach is close to \cite{Ono2017}, where a PnP-PDS algorithm is shown to have similar benefits. The convergence of PDS, though, has never been established in the non-convex or PnP case dealt with in this paper.
In contrast, the proposed method is shown to converge
to a critical point of $E(x)=h(Hx)+\lambda f(x)$ under less restrictive conditions than in previous works on PnP-ADMM, which require the denoiser residual to be Lipschitz continuous \cite{pmlr-v97-ryu19a,Hurault2022}, and impose constraints on the regularization parameter $\lambda$ \cite{pmlr-v97-ryu19a}. Such constraints on $\lambda$ mean that we need to choose between convergence guarantees and optimal regularization.
As for the denoiser constraints, several techniques exist to train a Lipschitz denoiser, but at the cost of degraded denoising performance~\cite{Hurault2022}.
The Linearized PnP-ADMM that we introduce in the next section does not require the denoiser to be Lipschitz continuous nor does it impose any constraints on $\lambda$.

\section{Model}
\label{sec:model}
In this section, we introduce our Plug \& Play linearized-ADMM algorithm (PnP LADMM). We first describe the main difference between ADMM and linearized-ADMM before discussing the convergence of linearized-ADMM in the case of Plug \& Play.

\subsection{Linearized-ADMM (LADMM)}

In order to solve MAP estimation problems like \eqref{equ:map-optim} ADMM starts from the augmented Lagrangian
\begin{align}
\label{equ:lagrangian}
    \mathcal{L}_\beta(x,z,w) = & h(z) + \lambda f(x) + \langle w, Hx-z\rangle \nonumber \\
    & + \frac{\beta}{2}\|Hx-z\|^2.
\end{align}
Note that in our case we used the splitting variable $Hx=z$, instead of the more common choice $x=z$ \cite{pmlr-v97-ryu19a, Hurault2022}, which leads to the potentially expensive computation of $\prox_{\alpha h(H \cdot)}$.
ADMM is based on a alternate minimization on the three variables of the Lagrangian \eqref{equ:lagrangian}, namely
\begin{align}
    & x_{k+1} = \arg\min_x \mathcal{L}_\beta(x,z_k,w_k) \\
    & z_{k+1} = \arg\min_z \mathcal{L}_\beta(x_{k+1},z,w_k) \\
    & w_{k+1} = w_k + \beta (Hx_{k+1}-z_{k+1}).
\end{align}
Now the $z$-update only requires the simpler computation of $\prox_{\alpha h}$, but the $x$-update is intractable because it involves both $f$ and $H$. The main idea of linearized-ADMM is to replace the minimization of the Lagrangian in the $x$-update by the minimization of an approximate or "linearized" Lagrangian where the quadratic term $\frac{\beta}{2}\|z-Hx\|^2$ is replaced by an isotropic majorizer with curvature $L_x \geq \beta \LH$:
\begin{align}
\label{equ:approx_lagr}
    \Tilde{\mathcal{L}}^k_{\beta}(x,z,w) = & h(z) + \lambda f(x) + \langle w, Hx-z\rangle + \frac{L_x}{2}\|x-x_k\|_2^2 \nonumber \\
    & + \frac{\beta}{2}\langle x-x_k, 2 H^T(Hx_k-z) \rangle.
\end{align}
Using this notation, we can express linearized-ADMM as:
\begin{align}
    \label{equ:ladmm_x}
    & x_{k+1} =  \arg\min_x{\Tilde{\mathcal{L}}^k_{\beta}(x,z_k,w_k)} \\
    \label{equ:ladmm_z}
    & z_{k+1} = \arg\min_z{\mathcal{L}_\beta(x_{k+1},z,w_k)}  \\
    \label{equ:ladmm_w}
    & w_{k+1} = w_k + \beta(Hx_{k+1} - z_{k+1}).
\end{align}
\subsection{Convergence}
Despite the approximation we can show that LADMM converges to the expected critical point under mild assumptions.
\begin{assumption}
\label{assump:1}

\begin{itemize}
    \item $h(z) + \lambda f(x)$ is lower bounded on the set $\{(z,x) \in (\mathbb{R}^{n*p})^2| z = Hx\}$.
    \item $h$ is strongly convex and $L_h$-Lipschitz differentiable 
\end{itemize}
\end{assumption}
\begin{theorem}
\label{theorm1}
 Under Assumption~\ref{assump:1}, for linearized-ADMM with hyper parameters such that:
 \begin{align}
 \label{eq:convergence-condition1}
     & \beta \geq L_h \\
\label{eq:convergence-condition2} 
    & L_x \geq \beta \LH
 \end{align}
 then the sequence $\{\mathcal{L}_\beta(x_k,z_k,w_k)\}$ is convergent and the primal residues $\|x_{k+1}-x_k\|$, $\|z_{k+1}-z_k\|$ and the dual residue $\|w_{k+1}-w_k\|$ converge to 0 as k approaches infinity. \newline
We also have that the sequence ${(x_k,z_k,w_k)}$ satisfies 
\begin{equation}
\lim_{k\to\infty} \nabla_w \mathcal{L}_\beta(x_k,z_k,w_k)=
\lim_{k\to\infty} \nabla_z \mathcal{L}_\beta(x_k,z_k,w_k)=0
\end{equation}
and that there exists
\begin{equation}
    d^k \in \partial_x \mathcal{L}_\beta(x_k,z_k,w_k) \quad \text{s.t} \quad \lim_{k\to\infty} d^k = 0.
\end{equation}
If in addition $f$ is differentiable then 
%\begin{equation}
$\lim_{k\to\infty} \nabla E(x_k)=0$.\footnotemark[3]
%\end{equation}
\end{theorem}
\footnotetext[3]{A proof of Theorem~\ref{theorm1} (adapted from~\cite{Liu2019}) and of Proposition~\ref{prop:pnp-denoisers} (based on \cite{Gribonval2011} and \cite{Hurault2022} is provided \ProofLocation}

\subsection{Plug \& Play linearized-ADMM (PnP-LADMM)}

Using the change of variable $u_k = \frac{w_k}{\beta}$ and re-aranging the terms in the optimization steps from equation (\ref{equ:ladmm_x}-\ref{equ:ladmm_w}), we can obtain the proximal version of linearized ADMM:
\begin{align}
\label{equ:ladmm_prox_x}
    & x_{k+1} = \prox_{\frac{\lambda}{\L_x}f}(x_k - \frac{\beta}{L_x}H^T(Hx_k - z_k + u_k)) \\
    & z_{k+1} = \prox_{\frac{1}{\beta}h}(H x_{k+1} + u_k)\\
    & u_{k+1} = u_k + (Hx_{k+1} - z_{k+1}).
\end{align}
The proximal operator in (\ref{equ:ladmm_prox_x}) can be seen as a denoising problem with regularization function $f$ and noise level $\sigma_d^2 = \frac{\lambda}{L_x}$. In the spirit of Plug \& Play approaches, this proximal operator can be replaced by an off-the-shelf denoiser $\mathcal{D}_{\sigma_d}$ (see Algorithm~\ref{algo}). In comparison to PnP-ADMM \cite{pmlr-v97-ryu19a,Hurault2022} which requires the denoiser residual $\mathcal{D}_{\sigma_d}-Id$ to be non-expansive to ensure convergence, the proposed PnP-LADMM converges for a larger family of denoisers:
\begin{proposition}\label{prop:pnp-denoisers}
If $\mathcal{D}_{\sigma_d}$ is any MMSE denoiser or the Proximal Gradient Step denoiser in \cite{Hurault2022}, then there exists a lower bounded function $f$ such that $\mathcal{D}_{\sigma_d} = \prox_{\sigma_d^2f}$.\footnotemark[3]
\end{proposition}
As a consequence Theorem~\ref{theorm1} ensures convergence of PnP-LADMM for the gradient step denoiser or any MMSE denoiser. So we can use any state-of-the-art denoising architecture trained with quadratic loss for $\mathcal{D}_\sigma$. In practice, we adopt the widely used DRUNet denoiser that was introduced in~\cite{zhang2021plug}.
\begin{algorithm}[btph]
\caption{PnP Linearized ADMM algorithm
\\
{\small
Solves $x = \arg\min_x h(Hx)+\lambda f(x)$
}%end small
}\label{alg:pnp-sr}\label{algo}
\begin{algorithmic}
\Require $x_0,\, z_0,\, u_0,\, \beta,\, L_x,\, \mathcal{D}_{\sigma_d} = \prox_{\sigma_d^2 f}$
\For{$k \in [0, N-1]$}
    \State $x_{k+1} = \mathcal{D}_{\sqrt{\frac{\lambda}{L_x}}}(x_k-\frac{\beta}{L_x}H^T(Hx_k - z_k + u_k))$
    \State $z_{k+1} = \frac{y + \sigma^2 \beta (Hx_{k+1}+u_k)}{1 + \beta \sigma^2}$
    \State $u_{k+1} = u_k + \beta(z_{k+1} - Hx_{k+1})$
\EndFor
\end{algorithmic}
\end{algorithm}
The computational efficiency of the method relies both on the use of the splitting variable $Hx=z$ and the linearization. Using the splitting variable $Hx=z$ leads to a $z$-update that is very easy to compute since it corresponds to the proximal operator of a quadratic norm which does not involve the degradation operator $H$. On the other hand, the linearization leads to an $x$-update that bypasses the inversion of the degradation operator $H$. Our PnP-LADMM algorithm only requires that we can efficiently compute the quantities $Hx$ and $H^Tx$ at each iteration. The forward and adjoint of the degradation operator can be efficiently computed for a wide diversity of tasks such as super-resolution, spatially-varying blur, inpainting, compressed sensing, etc. The whole iterative process with the closed form formulation is summarized in Algorithm~\ref{alg:pnp-sr}.
\newline
\newline
PnP-LADMM has 4 different hyperparameters, $\lambda, \sigma_d, \beta$ and $L_x$. The parameters $\lambda$ and $\sigma_d$ are model parameters of the MAP estimator, they will be responsible for the quality of the output and control the balance between data fidelity and regularization. On the other hand, $\beta$ and $L_x$ are parameters of the optimization algorithm, they control the convergence speed. Since these 4 parameters are linked to each other via the constraint $\sigma_d = \sqrt{\lambda/ L_x}$, there are only 3 degrees of freedom for our algorithm. For a Gaussian data fitting term, we have $h(x) = \frac{1}{2\sigma^2}\|x-y\|_2^2$ so $L_h = 1/\sigma^2$. The condition of Theorem 1 implies that:
\begin{align}
    & L_x \geq \beta\|H\|^2 \geq L_h\|H\|^2 \\
    & \Leftrightarrow \lambda \geq \sigma_d^2 \beta\|H\|^2 \geq \frac{\sigma_d^2}{\sigma^2}\|H\|^2
\end{align}
This means that we can choose any non-negative regularization parameter $\lambda>0$ as long as we decrease $\sigma_d$ accordingly for very small values of $\lambda$.

\section{Experiments}
\label{sec:experiments}

In this section, we evaluate the performance of our approach on deblurring images with spatially-varying blur. All the code used in our experiments can be found in the \href{https://github.com/claroche-r/PnP_LADMM}{GitHub page} of the project.

\subsection{Datasets}
\begin{figure}[t]
\begin{subfigure}[t]{0.40\linewidth}
    \centering
    \includegraphics[width=\linewidth]{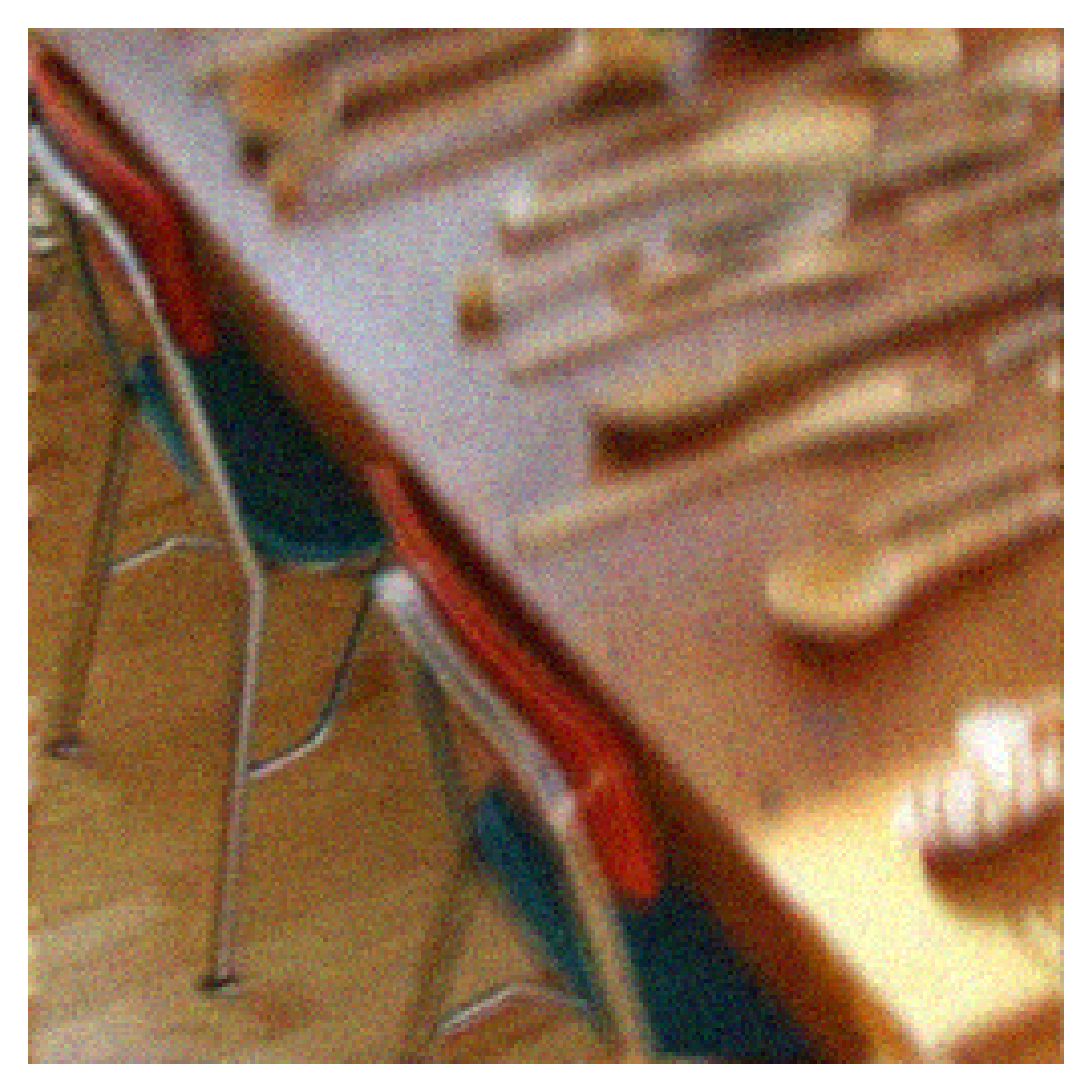}
    \caption{LR}
    \label{fig:one}
\end{subfigure}
\hspace*{\fill}
\begin{subfigure}[t]{0.40\linewidth}
    \centering
    \includegraphics[width=\linewidth]{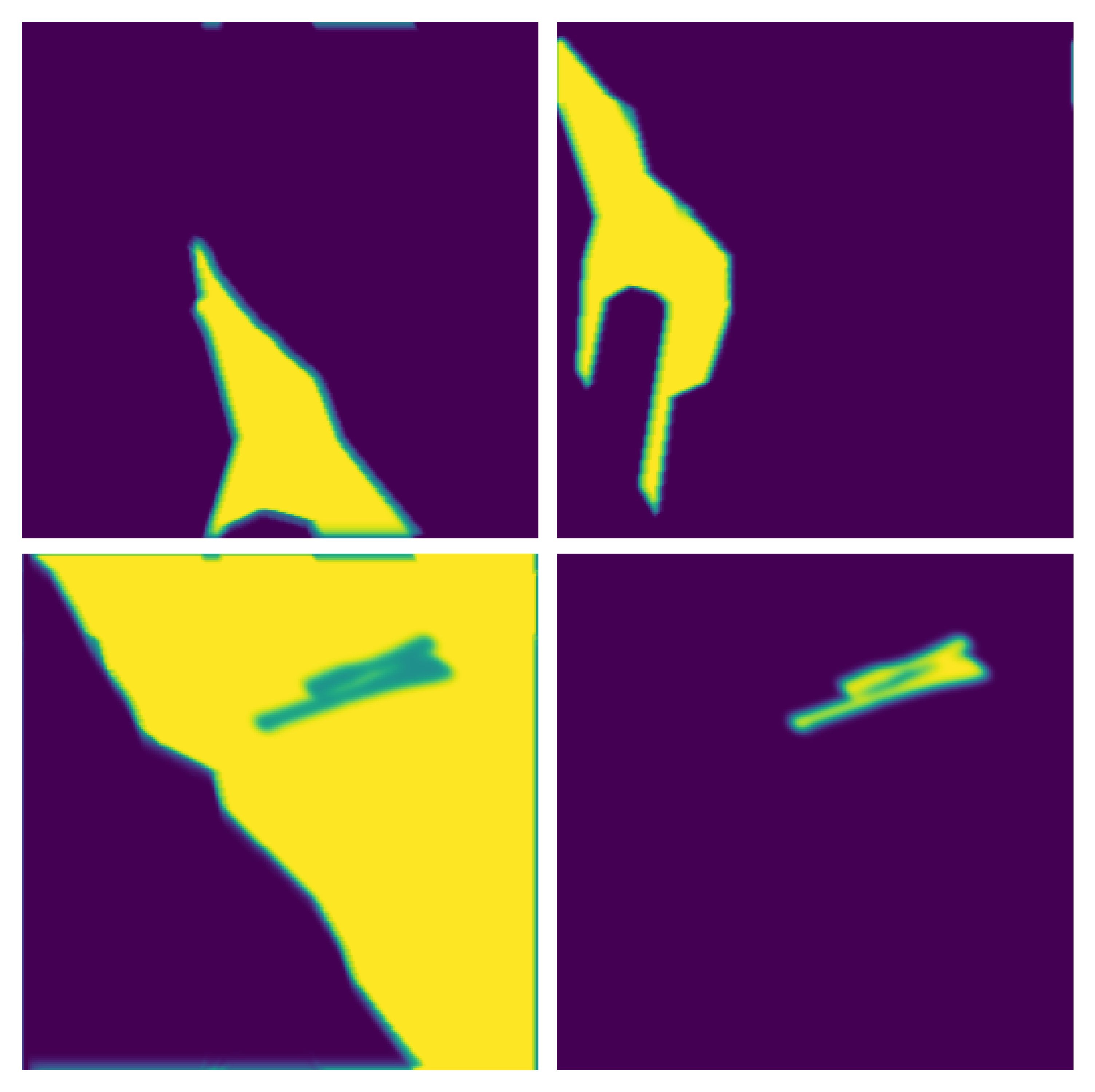}
    \caption{$U_i$'s}
    \label{fig:second}
\end{subfigure}
\vfill
\begin{subfigure}{\linewidth}
    \centering
    \includegraphics[width=0.40\linewidth]{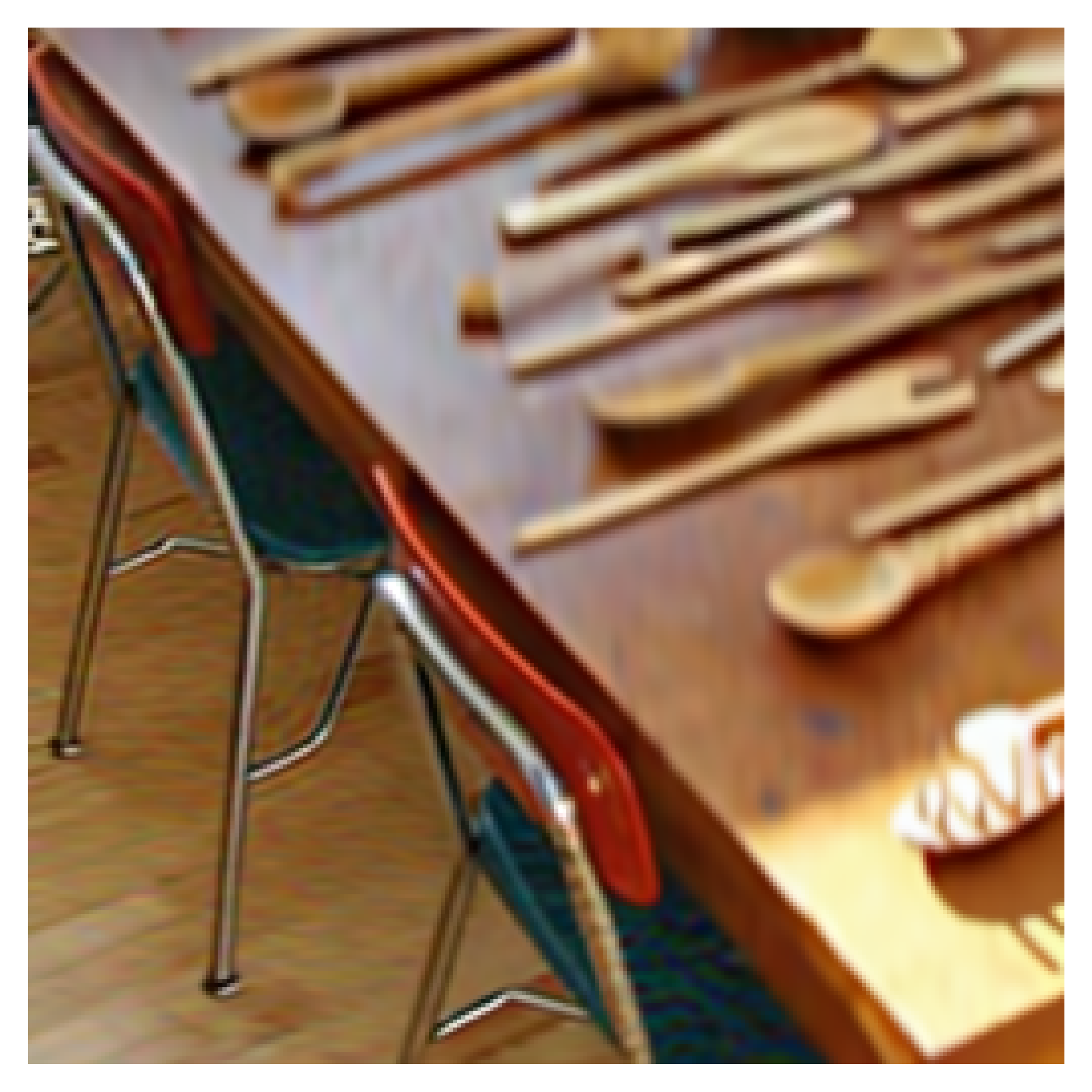}
    \caption{Deblurred}
    \label{fig:third}
\end{subfigure}
\caption{Example of sample from the testset with restoration result.}
\label{fig:samples-o-leary}
\end{figure}
\begin{table}[t]
    \centering
    \resizebox{\columnwidth}{!}{%
    \begin{tabular}{|c|c|c|c|}
        \hline
        Model & Runtime  &  $\sigma$ &   Metrics  \\
        \hline
        & &  &   (PSNR$\uparrow$ , SSIM$\uparrow$ , LPIPS$\downarrow$ )  \\
        \hline
        \hline
        \multirow{4}{*}{Richardson-Lucy}&  \multirow{4}{*}{10sec} & 1 &  (23.4, 0.74, 0.27) \\
         & & 10 & (20.9, 0.43, 0.55)	\\
         & & 20 & (18.8, 0.25, 0.64)	 \\
         & & 40 & (15.4, 0.13, 0.72)	\\
        \hline
        \multirow{4}{*}{PnP-ISTA} & \multirow{4}{*}{247sec} & 1 &  (23.4, 0.71, 0.34) \\
         & & 10 & (23.3, 0.71, 0.33) \\
         & & 20 & (22.7, \textbf{0.67}, 0.38)  \\
         & & 40 & (21.7, \textbf{0.61}, 0.43) \\
        \hline
        \multirow{4}{*}{PnP-ADMM + CG} & \multirow{4}{*}{286sec} & 1 &  (\textbf{25.8, 0.82}, 0.26)	 \\
         & & 10 & \textbf{(23.7, 0.72, 0.32)} \\
         & & 20 & \textbf{(22.9, 0.67, 0.37)} \\
         & & 40 & (\textbf{21.7}, 0.60, \textbf{0.43}) \\
        \hline
        \multirow{4}{*}{PnP-LADMM} &   \multirow{4}{*}{124sec} & 1  &  (25.6, 0.81, \textbf{0.22}) \\ 
         & & 10 & (\textbf{23.7}, \textbf{0.72, 0.32}) \\
         & & 20 & (22.8, 0.66, 0.38) \\
         & & 40 & (\textbf{21.7}, \textbf{0.61}, \textbf{0.43}) \\
        \hline
    \end{tabular}%
    }
    \caption{Performance of the different models, PnP-ADMM + CG refers to PnP-ADMM where the proximal operator of the data term is computed using conjugate gradient algorithm. Best results are in \textbf{bold}.}
    \label{tab:quantitative-res}
\end{table}
We test our approach on deblurring images with non-uniform blur. Non-blind deblurring algorithms usually suppose the blur to be uniform since it leads to easier computations both for the generation of synthetic data and for the deblurring. However, the uniform blur assumptions does not hold for many real-world applications, such as motion blur or defocus blur. To highlight the performance of our algorithm on deblurring spatially-varying blur, we degrade our images using the dataset introduced in~\cite{DMBSR}  which uses the O'Leary~\cite{oleary} model. In particular, we suppose that the blur $H$ in the inverse problem~\eqref{equ:inverse-problem}
is decomposed as a linear combination
\begin{equation}
     H = \sum\limits_{i=1}^P {U_i K_i},  \ \varepsilon \sim \mathcal{N}(0, \sigma^2)
 \end{equation}
of uniform blur (convolution) operators $K_i$ with spatially varying mixing coefficients, \emph{i.e.} diagonal matrices $U_i$ such that $\sum\limits_{i=1}^P{U_i} = Id$, $U_i \geq 0$. Please note that even with this decomposition, the proximal operator $\prox_{h(H.)}$ from Equation \eqref{equ:prox_g} cannot be easily and efficiently computed. The advantage of the O'Leary model is that its forward and transpose operators can be computed very efficiently using convolution and masking operations. Also, this formulation can model a large diversity of spatially varying blurs. In our experiments, we apply this degradation process on the COCO dataset~\cite{coco}, using the segmentation masks as the $U_i$'s and we build random Gaussian and motion blur kernels for the $K_i$'s. Figure~\ref{fig:samples-o-leary} shows the low-resolution obtained.

\subsection{Compared methods}

We compare our approach to the Richardson-Lucy algorithm~\cite{richardson1972, lucy1974}, Plug \& Play ADMM with splitting variable $x=z$ \cite{pmlr-v97-ryu19a} where the proximal operator is approximated using conjugate gradient algorithm (PnP-ADMM + CG) and Plug \& Play ISTA~\cite{gavaskar2020} (PnP-ISTA). These algorithms are common for deblurring with spatially varying kernels (see \cite{debarnot2022} for PnP-ADMM+CG and \cite{carbajal2021nonuniform} for Richardson-Lucy) and the convergence of PnP-ADMM and PnP-ISTA has been proven.
For a fair comparison, we use the same DRUNet denoiser pretrained on Gaussian noise with $\ell^2$ loss for all the methods.
This choice ensures that the denoiser is a good approximation of an MMSE estimator, which is sufficient to guarantee convergence for our PnP-LADMM algorithm, and for PnP-ISTA as well \cite{Xu2020}. In order to ensure convergence of PnP-ADMM we could have used a modified denoiser with non-expansive residual, but this constraint degrades its performance and it is most often not necessary to ensure convergence in practice.
We observed that all methods do not behave the same for similar denoiser noise level $\sigma_d$ and regularization parameter $\lambda$. Following this observation, we decided to separately tune these parameters for each method and dataset noise level to find their respective optima in terms of restoration quality on a small training dataset. Results are then evaluated on a different test dataset of 40 images with the optimal parameters found on the training set. 
\begin{figure}[t]
    \centering
    \includegraphics[width=0.9\linewidth]{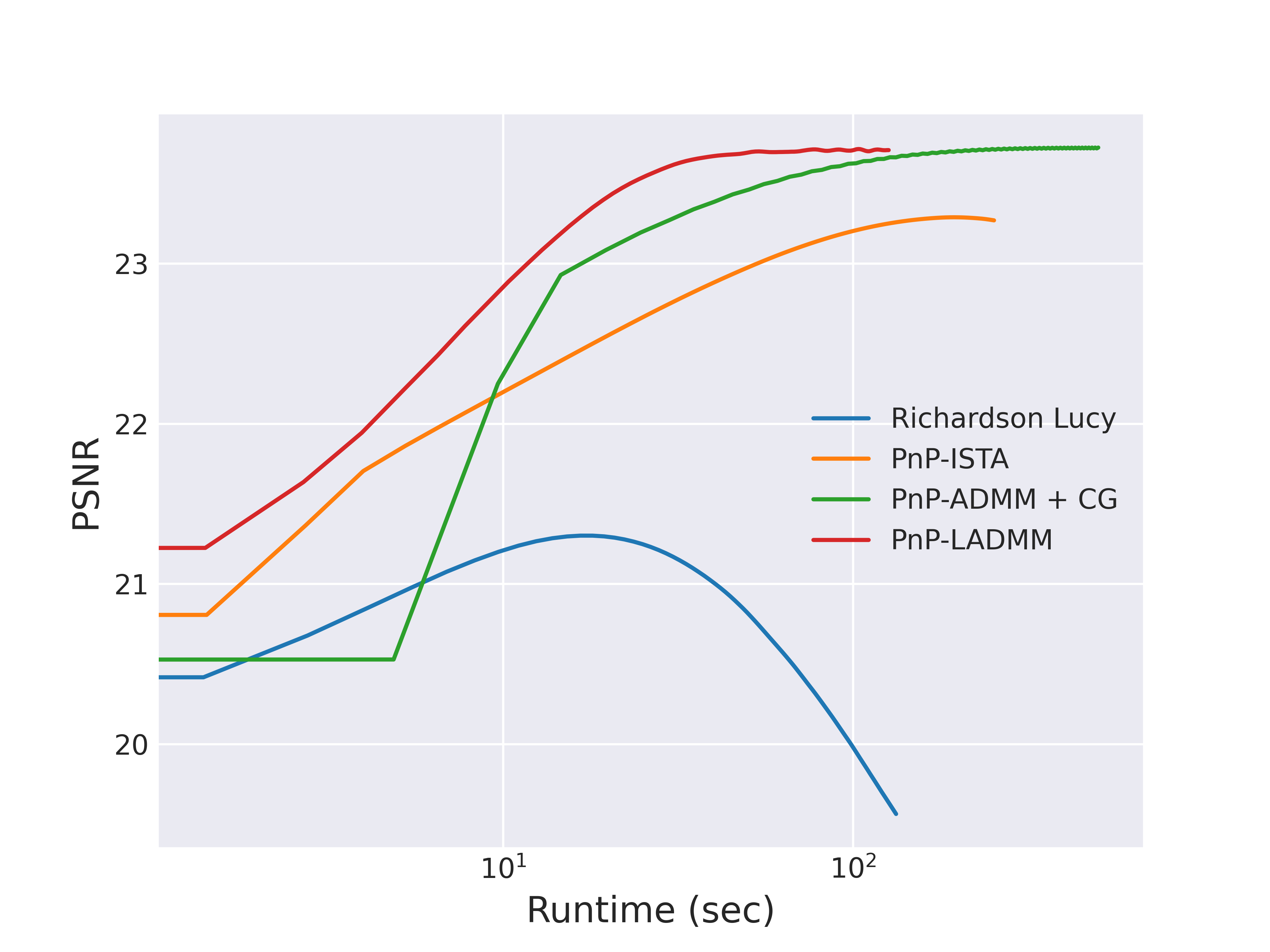}
    \caption{Convergence speed of the different methods, we use 40 images with spatially-varying blur and Gaussian noise with $\sigma=10/255$.}
    \label{fig:plot_speed}
\end{figure}

\subsection{Results}
        
We evaluate the performance of the models using classical metrics such as PSNR, SSIM and LPIPS. The overall performance results are summarized in Table~\ref{tab:quantitative-res}.
In terms of performances, the baseline Richardson-Lucy fails to deblur the noisy images. PnP-LADMM and PnP-ADMM + CG have similar performances with a maximum difference of 0.3dB in favor of the latter. PnP-ISTA has slightly lower performances than the two ADMM based algorithms. It seems that the margin is decreasing when the noise level increases.
In terms of runtime, similar observation are found in Table~\ref{tab:quantitative-res} and Figure~\ref{fig:plot_speed}. Richardson-Lucy quickly reaches its optimum since it does not involve expensive operations. PnP-ISTA is twice slower than PnP-LADMM. As discussed in the introduction, gradient-based methods are often slow to converge. In particular here, PnP-ISTA requires around 200 iterations to converge in comparison to 100 for PnP-LADMM. Since one PnP-ISTA iteration takes as much time as one PnP-LADMM iteration (we apply one denoiser forward, one $H^Tx$ and one $Hx$ operation), it results in an algorithm that is slower. PnP-ADMM + CG is the slowest algorithm despite needing only 40 iterations to converge. This slowness is due to the fact that the conjugate gradient is computed at each step in order to approximate the proximal operator. Finally, PnP-LADMM is the fastest PnP algorithm. In fact, even though the linearization is causing the algorithm to converge in more iterations, the benefit of bypassing the computation of the proximal operator is greater resulting in a faster algorithm.
These results highlight the fact that for complex degradation operators, PnP-LADMM achieves the best performance/ratio trade-off. 

\section{Conclusion}
\label{sec:conclusion}
We presented a novel Plug \& Play approach based on LADMM to solve inverse problems with complex degradation operators such as non-uniform blur. The linearized version of ADMM allows to bypass the computation of intractable proximal operators. We demonstrate the efficiency of our method on the problem of deblurring images with O'Leary spatially-varying blur. We found that PnP-LADMM obtains the best performance/runtime trade-off compared to the gradient-based method PnP-ISTA or PnP-ADMM combined with conjugate gradient to compute the proximal operator. In addition, the proposed algorithm provides convergence guarantees under less restrictive conditions than previous PnP-ADMM results.

\bibliographystyle{IEEEbib}
\bibliography{refs}

\ifthenelse{\boolean{WithProof}}{\newpage\onecolumn
\appendix
\section{Convergence of linearized-ADMM}

In this section, we prove Theorem~\ref{theorm1} and Proposition~\ref{prop:pnp-denoisers}.
The proof of Theorem~\ref{theorm1} is an adaptation of a similar result in~\cite{Liu2019}. The main difference is that in~\cite{Liu2019} the ADMM is linearized in both proximal descent steps, whereas in our case we are interested in linearizing only one of them.

In the sequel, we suppose (without loss of generality) that $\lambda=1$ in our MAP estimator defined in~(\ref{equ:map-optim}).

%
% Lagrangian inequality in x
%
\begin{lemma}
\label{lemma:lagr_x}
    Under Assumption 1, the following inequality holds for the x-update:
    $$\mathcal{L}_\beta(x_k,z_k,w_k) - \mathcal{L}_\beta(x_{k+1},z_k,w_k) \geq \frac{L_x - \beta \LH}{2}\|x_k - x_{k+1}\|^2 $$
    with $\LH$ the largest singular value of $H^T H$.
\end{lemma}
\begin{proof}
    Using the notation of equation~\eqref{equ:approx_lagr}, we define $\overline{f}^k (x) = \Tilde{\mathcal{L}}^k_{\beta}(x,z_k,w_k)$ and by definition of the $x$-update, we have:
    \begin{align}
        & \overline{f}^k(x_k) \geq \overline{f}^k(x_{k+1}) \\
        \Leftrightarrow & \langle x_k - x_{k+1} , H^T w_k + \beta H^T(Hx_k - z_k)\rangle \nonumber \\
         \label{equ:fk}
         & + f(x_k) - f(x_{k+1}) \geq   \frac{L_x}{2}\|x_{k+1}-x_k\|^2
    \end{align}
    We also have that:
    \begin{align}
        & \mathcal{L}_\beta(x_k,z_k,w_k) - \mathcal{L}_\beta(x_{k+1},z_k,w_k)  \nonumber \\
        & =f(x_k) - f(x_{k+1}) + \langle w_k, H(x_k-x_{k+1}) \rangle \nonumber \\
        & \ \  + \frac{\beta}{2}\|Hx_k-z_k\|^2 - \frac{\beta}{2}\|Hx_{k+1} - z_k\|^2  \\
        &=  f(x_k) - f(x_{k+1}) - \frac{\beta}{2}\|H(x_{k+1} -x_k) \|^2 \nonumber \\
        & \ \ + \langle x_k - x_{k+1}, H^T w_k + \beta H^T(Hx_k - z_k) \rangle \\
        & \geq  \frac{L_x}{2}\|x_{k+1}-x_k\|^2 - \frac{\beta}{2}\|H(x_{k+1} -x_k) \|^2  \label{equ:lemma_3_1}  \\
        & \geq \frac{L_x-\beta \LH}{2}\|x_{k+1} -x_k\|^2
    \end{align}
    where the inequality \eqref{equ:lemma_3_1} is obtained using \eqref{equ:fk}.
\end{proof}
%
% Lagrangian inequality in z
%
\begin{lemma}
\label{lemma:lagr_z}
    $\mathcal{L}_\beta(x_{k+1},z_k,w_k) - \mathcal{L}_\beta(x_{k+1},z_{k+1},w_k) \geq m\|z_k - z_{k+1}\|^2$
\end{lemma}
\begin{proof}
    From Assumption 1, we have that $\mathcal{L}_\beta$ is strongly convex in $z$ with parameter $m$. The strong convexity of $\mathcal{L}_\beta$ implies that:
    \begin{align}
        & \mathcal{L}_\beta(x,z_k,w) - \mathcal{L}_\beta(x,z_{k+1},w) \\
        & \ \ \ \ \geq \nabla_z \mathcal{L}_\beta(x,z_{k+1},w)(z_k - z_{k+1}) + m \|z_k - z_{k+1}\|^2 
    \end{align}
    However, the $z$-update of Algorithm \ref{algo} is such that
    \begin{equation}
        \nabla_z \mathcal{L}_\beta(x_{k+1},z_{k+1},w_k) = 0
    \end{equation}
    which leads to the results.
\end{proof}
\begin{lemma}
\label{lemma:w_grad_h}
    Under Assumption~\ref{assump:1}, the following equality holds:
    \begin{equation}
        w_k = \nabla_z h(z_k)
    \end{equation}
\end{lemma}
\begin{proof}
    From the definition of the Lagrangian:
    $$\nabla_z \mathcal{L}_\beta(x,z,w) = \nabla h(z) - w - \beta (Hx-z)$$
    Using the fact that 
    \begin{equation}
       w_{k+1}= w_k + \beta (Hx_{k+1}-z_{k+1}) \quad \text{and} \quad \nabla_z \mathcal{L}_\beta(x_{k+1},z_{k+1},w_k) =0
    \end{equation}
    We have:
    \begin{align}
        & 0 =  \nabla h(z_{k+1}) - w_k - \beta (Hx_{k+1}-z_{k+1}) \\
        \Leftrightarrow &  \nabla h(z_{k+1}) = w_{k+1}
    \end{align}
\end{proof}
\begin{lemma}
\label{lemma:lagr_u}
Under assumption 1, 
\begin{align}
   & \mathcal{L}_\beta(x_{k+1},z_{k+1},w_{k+1}) - \mathcal{L}_\beta(x_{k+1},z_{k+1},w_k)  \\
   & =  \frac{1}{\beta}\|w_{k+1} - w_k\|^2  \leq  C_1\|z_{k+1}-z_k\|^2
\end{align}
with $C_1 = L_h^2/\beta$.
\end{lemma}
\begin{proof}
    By definition of the augmented Lagrangian we have that:
    \begin{align*}
        \mathcal{L}_\beta(x_{k+1},z_{k+1},w_{k+1}) & - \mathcal{L}_\beta(x_{k+1},z_{k+1},w_k) \\
        & = \langle w_{k+1} - w_k, H x_{k+1}  - z_{k+1}\rangle \\
        & = \frac{1}{\beta}\|w_{k+1}-w_k\|^2 \\
        & = \frac{1}{\beta}\|\nabla_z h(z_{k+1}) - \nabla_z h(z_{k})\|^2 \ \ \text{ from Lemma~\ref{lemma:w_grad_h}} \\
        & \leq \frac{L_h^2}{\beta}\|z_{k+1}-z_k\|^2 \ \ \text{\hfill from \ Assumption~\ref{assump:1}}%Equation~\ref{equ:gradL} 
    \end{align*}
\end{proof}
\begin{lemma}
\label{lemma:lipsch}
    If $g$ is  $L_g$-Lipschitz differentiable then:
    \begin{equation}
        g(y_2) - g(y_1) \geq \nabla g(s)(y_2 - y_1) - \frac{L_g}{2}\|y_2-y_1\|^2 
    \end{equation}
    where s denotes $y_1$ or $y_2$
\end{lemma}
\begin{proof}

\begin{align}
& g(y_2)-g(y_1)
= \int_0^1 \nabla g(ty_2+(1-t) y_1) \cdot(y_2-y_1) \mathrm{d} t \\
=& \int_0^1 \nabla g(s) \cdot(y_2-y_1) \mathrm{d} t+\int_0^1(\nabla g(y_2+(1-t) y_1)-\nabla g(s))\cdot(y_2-y_1) \mathrm{d} t,
\end{align}
where $\nabla g(\cdot)$ defines the gradient of $g(\cdot)$. If we take $s=y_1$, then by inequality
\begin{equation}
\|\nabla g(t y_2+(1-t) y_1)-\nabla g(y_1)\| \leq L_g\|t(y_2-y_1)\|
\end{equation}
we have
\begin{align}
& \int_0^1 \nabla g(y_1) \cdot(y_2-y_1) \mathrm{d} t+\int_0^1(\nabla g(t y_2+(1-t) y_1)-\nabla g(y_1)) \cdot(y_2-y_1) \mathrm{d} t \\
\geq & \nabla g(y_1) \cdot(y_2-y_1)-\int_0^1 L_g t\|y_2-y_1\|^2 \mathrm{~d} t \\
=& \nabla g(y_1) \cdot(y_2-y_1)-\frac{L_g}{2}\|y_2-y_1\|^2 .
\end{align}
Therefore, we get
\begin{equation}
g(y_2)-g(y_1) \geq \nabla g(y_1) \cdot(y_2-y_1)-\frac{L_g}{2}\|y_2-y_1\|^2 .
\end{equation}
Similarly, if we take $s=y_2$, we can get
\begin{equation}
g(y_2)-g(y_1) \geq \nabla g(y_2) \cdot(y_2-y_1)-\frac{L_g}{2}\|y_2-y_1\|^2 .
\end{equation}
\end{proof}
\begin{lemma}\label{lemma:mk-converges}
    Under Assumption~\ref{assump:1}, if we choose the hyper-parameters $\beta$ and $L_x$ satisfying (\ref{eq:convergence-condition1}) and \eqref{eq:convergence-condition2}, then the sequence $\{m_k\}$ defined by
    \begin{equation}
        m_k = \mathcal{L}_\beta(x_k,z_{k},w_{k})
    \end{equation}
    is convergent.
\end{lemma}
\begin{proof}
    1) \textbf{Monotonicity:} By using Lemma~\ref{lemma:lagr_x}, Lemma~\ref{lemma:lagr_z} and Lemma~\ref{lemma:lagr_u} we have:
    \begin{align}
        & m_{k} - m_{k+1} = \mathcal{L}_\beta(x_{k}, z_{k}, w_{k}) - \mathcal{L}_\beta(x_{k+1}, z_{k+1}, w_{k+1}) \\
        & \ \ \ \ \geq \mathcal{L}_\beta(x_{k+1}, z_{k}, w_{k}) - \mathcal{L}_\beta(x_{k+1}, z_{k+1}, w_{k+1}) \\
        & \ \ \ \ \ \ \ +  \frac{L_x-\beta \|H\|^2}{2} \|x_k - x_{k+1}\|^2 \\
        & \ \ \ \ \geq \mathcal{L}_\beta(x_{k+1}, z_{k+1}, w_{k}) - \mathcal{L}_\beta(x_{k+1}, z_{k+1}, w_{k+1}) \\
        & \ \ \ \ \ \ \ +  \frac{L_x-\beta\LH}{2} \|x_k - x_{k+1}\|^2  + m \|z_{k}-z_{k+1}\|^2 \\
        & \label{equ:monot_mk_1}
        \ \ \ \ \geq  \frac{L_x-\beta\LH}{2} \|x_k - x_{k+1}\|^2 + (m + \frac{L_h^2}{\beta})\|z_{k}-z_{k+1}\|^2
    \end{align}
    Since we chose $L_x$ such that:
    \begin{align}
        & L_x \geq  \beta\LH \\
        \Leftrightarrow \quad  & \frac{L_x-\beta\LH}{2} > 0
    \end{align}
    we obtain the monotonocity of $\{m_k\}$.
    \newline
    \newline
    2) \textbf{Lower bound}:
    \begin{align}
        m_k & =  h(z_k) + f(x_k) + \langle w_k, H x_k-z_k\rangle + \frac{\beta}{2}\|Hx_k-z_k\|^2
    \end{align}
    Let $z_{k}' = H x_k$, from Lemma~\ref{lemma:w_grad_h} we have:
    \begin{align}
        \langle w_k, H x_k-z_k\rangle 
        & = \langle w_k, z_{k}' - z_k\rangle \\
        & = \langle \nabla h(z_k) , z_{k}' - z_k\rangle
    \end{align}
    so we can rewrite:
    \begin{align}
        m_k & =  h(z_k) + f(x_k) + \langle \nabla h(z_k) , z_k' - z_k\rangle + \frac{\beta}{2}\|z_k'-z_k\|^2 \\
    \end{align}
    We chose $\beta$ such that $\beta \geq L_h$ so:
    \begin{align}
        m_k & \geq  h(z_k) + f(x_k) - \langle \nabla h(z_k) , z_k - z_k'\rangle + \frac{L_h}{2}\|z_k - z_k'\|^2 \\
        & \geq h(z_k') + f(x_k) \ \ \text{from Lemma~\ref{lemma:lipsch}.}
    \end{align}
    Following Assumption~\ref{assump:1}, $h(z_k') + f(x_k)$ is lower bounded so ${m_k}$ is lower bounded. ${m_k}$ is monotonically decreasing and lower bounded which ensures the convergence. 
\end{proof}
\begin{lemma}
\label{lemma_f1_f2}
    Suppose we have a differentiable function $f_1$, a possibly non differentiable function $f_2$, and a point x. If there exist $d_2 \in \partial f_2(x)$, then we have:
    $$ d=d_2 - \nabla f_1(x) \in \partial(f_2(x) - f_1(x)) $$
\end{lemma}
\begin{proof}
    From the subgradient definition we have that:
    \begin{equation}
        f_2(y) \geq f_2(x) + \langle d_2, y-x \rangle + o(\|y-x\|)
    \end{equation}
    From the fact that $f_1$ is differentiable we have that:
    \begin{equation}
        -f_1(y) = -f_1(x) - \langle\nabla f_1(x), y-x\rangle + o(\|y-x\|)
    \end{equation}
    Combining the two leads to:
    \begin{equation}
        f_2(y) -f_1(y) \geq f_2(x) - f_1(x) + \langle d_2 - \nabla f_1(x), y-x\rangle + o(\|y-x\|)
    \end{equation}
\end{proof}
\noindent
\begin{proof}[Proof of Theorem 1:] We divide the proof in three parts:

\textbf{a) Convergence of the residuals:} 
\newline
From Lemma~\ref{lemma:mk-converges} and its proof we have that:
    \begin{equation}
        m_{k+1} - m_k \geq a \|x_k - x_{k+1}\|^2 + \left(m+\frac{L_h^2}{\beta}\right) \|z_{k-1}-z_{k}\|^2 \geq 0
    \end{equation}
    with $(m + \frac{L_h^2}{\beta}) > 0$, $a = \frac{L_x-\beta \LH}{2} >0$ (according to Assumption 1)
    and that $m_k$ converges. This implies that $\|x_k - x_{k+1}\|^2$ and $\|y_k - y_{k+1}\|^2$ converge to 0 as $k$ approaches infinity. Lemma~\ref{lemma:lagr_u} ensure the convergence of $\|w_k - w_{k+1}\|^2$ to 0. The convergence of $m_k$ directly implies the convergence of $\mathcal{L}_\beta(x_k,z_{k},w_{k})$.
%    \newline

\textbf{b) Convergence of the gradients:}
\newline
For the convergence of $\lim_{k\to\infty} \nabla_w \mathcal{L}_\beta(x_k,z_k,w_k)$, we have that:
    \begin{equation}
        \lim_{k\to\infty} \nabla_w \mathcal{L}_\beta(x_k,z_k,w_k) =  \lim_{k\to\infty} Hx_k - z_k = \lim_{k\to\infty}\frac{1}{\beta}(w_{k+1}-w_k) =0.
    \end{equation}
On the other side, we have using Lemma~\ref{lemma:w_grad_h} that: 
\begin{align}
     \nabla_z \mathcal{L}_\beta(x_k,z_k,w_k) & = \nabla h(z_k) - w_k - \beta (Hx_k - z_k) \\
     & = w_k - w_k - (w_{k+1}-w_k) = -(w_{k+1}-w_k) \rightarrow 0
\end{align}
Finally, we want to show that there exists
\begin{equation}
    d^k \in \partial_x \mathcal{L}_\beta(x_k,z_k,w_k) \quad \text{s.t} \quad \lim_{k\to\infty} d^k = 0.
\end{equation}
Since $x^{k+1}$ is the minimum point of  $\Tilde{\mathcal{L}}^k_{\beta}(x,z_k,w_k)$, we have that $0 \in \partial \Tilde{\mathcal{L}}^k_{\beta}(x,z_k,w_k)$. Using Lemma~\ref{lemma_f1_f2} and the definition of $\Tilde{\mathcal{L}}^k_{\beta}$ we have:
\begin{align}
    &\exists d_{k+1} \in \partial f(x_{k+1}) \\
    \label{equ:dk}
    s.t \quad & H^T w_k + L_x(x_{k+1}-x_k) + \beta H^T (Hx_{k} -z_{k}) + d_{k+1} = 0
\end{align}
Lets us define:
\begin{equation}
    \Tilde{d}_{k+1} = H^T w_{k+1} + \beta H^T (Hx_{k+1} -z_{k+1}) + d_{k+1}
\end{equation}
we can easily verify that $\Tilde{d}_{k+1} \in \partial_x \mathcal{L}_\beta(x_{k+1}, z_{k+1}, w_{k+1})$.\newline
We already showed that the primal residues $\|x_{k+1}-x_k\|$, $\|z_{k+1}-z_k\|$, $\|w_{k+1}-w_k\|$ converge to 0 as k approaches infinity, therefore:
\begin{align}
    \lim_{k\to\infty}  \Tilde{d}_{k+1} &  =  \lim_{k\to\infty} H^T w_{k+1} + \beta H^T (Hx_{k+1} -z_{k+1}) + d_{k+1}\\
    & = \lim_{k\to\infty} H^T w_k + L_x(x_{k+1}-x_k) + \beta H^T (Hx_{k} -z_{k}) + d_{k+1} = 0
\end{align}
where the last equality is obtained using~\ref{equ:dk}.

\textbf{c) Convergence to a critical point of $h(H\cdot)+\lambda f(\cdot)$:} \newline
Since we are optimizing
\begin{equation}
E(x) = h(Hx) + f(x)
\end{equation}
we would like to show that
\begin{equation}
\lim_{k\to\infty} \nabla E(x_k) = 0
\end{equation}
Indeed, from the chain rule
\begin{equation}
\nabla E(x_k) = H^* \nabla h(Hx) + \nabla f(x).
\end{equation}
From part b of this proof we have that
\begin{align}
\lim_{k\to\infty} \nabla_w \L_\beta(x_k,z_k,w_k) 
& = \lim_{k\to\infty} z_k - H x_k 
& = 0 \label{eq:dLdw}\\
\lim_{k\to\infty} \nabla_z \L_\beta(x_k,z_k,w_k) 
& = \lim_{k\to\infty} w_k + \beta(z_k - Hx_k) + \nabla h(z_k)
& = 0 \label{eq:dLdz}\\
\lim_{k\to\infty} \nabla_x \L_\beta(x_k,z_k,w_k)
& = \lim_{k\to\infty} \nabla f(x_k) - H^* w_k + \beta H^*(Hx_k - z_k)
& = 0 \label{eq:dLdx}
\end{align}
where in the last equation we used the additional hypothesis that $f$ is differentiable.
Using equation \eqref{eq:dLdw} in equations~\eqref{eq:dLdz}~and~\eqref{eq:dLdx}:
\begin{align}
    \lim_{k\to\infty} z_k - H x_k & = 0 \label{eq:Hxk}\\
    \lim_{k\to\infty} \nabla h(z_k) + w_k & = 0 \\
    \lim_{k\to\infty} \nabla f(x_k) - H^* w_k & = 0 \label{eq:expr1}
\end{align}
Since $\nabla g$ is continuous, using equation \eqref{eq:Hxk} we get that
\begin{equation}\label{eq:expr2}
\lim_{k\to\infty} \nabla h(H x_k) + w_k = \lim_{k\to\infty} \nabla h(z_k) + w_k = 0 
\end{equation}
Rearranging the terms
\begin{align*}
    \nabla E(x_k) & = H^* \nabla h(H x_k) + \nabla f(x_k) \\
    & = H^* \nabla h(H x_k) + H^* w_k + \nabla f(x_k) - H^* w_k \\
    & = H^* (\nabla h(H x_k) + w_k) + (\nabla f(x_k) - H^* w_k)
\end{align*}
Finally using equations \eqref{eq:expr1} and \eqref{eq:expr2} in the previous result we obtain
$$ \lim_{k\to\infty} \nabla E(x_k) = H^*0 + 0 = 0$$
This shows that with the additional hypothesis of $f$ being differentiable, the Linearized ADMM converges to a critical point of the original objective $E(x)$.
\end{proof}

\section{Application to PnP-LADMM}

\subsection{Proof of Proposition~\ref{prop:pnp-denoisers}}

\subsubsection{Proximal Gradient Step Denoiser.} 
\begin{proof}
    Let $\denoiser_\sigmad$ be the proximal gradient step denoiser defined in \cite{Hurault2022} as $\denoiser_\sigmad := Id - \nabla g_\sigmad$ where 
$ g_\sigmad (x) = \frac{1}{2} \|x - N_\sigmad \|^2 $
and $N_\sigmad$ is a neural network.

According to \cite[Proposition 3.1]{Hurault2022} there exists $\phi_\sigmad$ such that $\denoiser_\sigmad = \prox_{\phi_\sigmad}$. 

In addition \cite[Equation (26)]{Hurault2022} states that $\phi_\sigmad \geq g_\sigmad$, and by definition $g_\sigmad \geq 0$.

Hence, $f=\phi_\sigmad / \sigma_d^2$ is lower bounded by 0 and $\denoiser_\sigmad = \prox_{\sigma_d^2 f}$ as indicated by Proposition~\ref{prop:pnp-denoisers}.
\end{proof}

\subsubsection{MMSE Denoiser.}
\begin{proof}
    Let $\denoiser_{\sigma_d}(y) = E[X|Y=y]$ be an MMSE denoiser, where $Y=X+\sigma_d N$ and $N\sim \mathcal{N}(0,\sigma_d^2 Id)$, and $X \sim p_X$, $p_X$ being a probability measure.
    
    We want to show that there exists a lower bounded $\phi_\sigmad$ such that $\denoiser_{\sigma_d}(x)=\operatorname{prox}_{\phi_\sigmad}(x)$.

    For $\sigma_d=1$ according to \cite{Gribonval2011} there exists $f(x) \geq - \log p_Y(x)$, such that $\denoiser_{1} = \operatorname{prox}_{f}$.
    $f$ is lower bounded because the noisy density $p_Y(x) = (p_X * g_1) (x) \leq 1/\sqrt{2\pi}$ is upper-bounded by the maximum value of $g_1$ (the gaussian pdf with identity covariance matrix).\\

    For $\sigma_d\neq 1$ the problem can be reduced to the previous case via the following scaling: Consider
    $\mathcal{P}(x) = \frac{1}{\sigma_d} \denoiser_{\sigma_d}(\sigma_d x)$.
    Then $\mathcal{P}(y) = E[\tilde{X}|\tilde{Y}=y]$ is an MMSE denoiser with variance 1 with $\tilde{X} = X/\sigma_d$ and $\tilde{Y}=\tilde{X} + N$.
    So we can find (according to the previous argument for $\sigma_d=1$) $f$ such that  $\mathcal{P} = \operatorname{prox}_{f}$.
    Applying a change of variables in the proximal operator we obtain
    $$ \denoiser_{\sigma_d} (y) = \sigma_d \mathcal{P}(y/\sigma_d) = \operatorname{prox}_{\phi_\sigmad}(y)$$
    where
    $$ \phi_\sigmad (x) = \sigma_d^2 f(x/\sigma_d) $$
    Finally, since $f$ is lower-bounded $\phi_\sigmad$ is lower bounded too.

\section{Additional experiments}

\subsection{Parameters influence}

PnP-linearized has 4 hyper-parameters that are $\sigma_d$, the noise removed by the denoiser, $\lambda$ the weight of the regularization, $\mu$ the ADMM parameter and $L_x$ the weight of the extra regularization introduced by linearized-ADMM. Those parameters are linked with the following equality:
\begin{equation}
    \sigma_d^2 = \frac{\lambda}{L_x}
\end{equation}
and through the condition from Equation~(\ref{eq:convergence-condition1}) and (\ref{eq:convergence-condition2}) of Theorem~\ref{theorm1}. Those equations can be summarized as:
\begin{align}
    & \frac{L_x}{\|H\|^2} \geq \beta \geq L_h = \frac{1}{\sigma^2} \\
    & \frac{\lambda}{\sigma_d^2 \|H\|^2} \geq \beta \geq \frac{1}{\sigma^2} 
%    & \frac{L_x}{\|H\|^2} \leq \beta \leq L_h = \frac{1}{\sigma^2} \\
%    & \frac{\lambda}{\sigma_d^2 \|H\|^2} \leq \beta \leq \frac{1}{\sigma^2} 
\end{align}

In practice, it means that for a fixed denoiser, there are only two degrees of freedom and that increasing $\lambda$ will lead to an increase on $L_x$. In Figure~\ref{fig:params_influence} we can observe the role of the different hyper-parameters for a fixed denoiser. We can observe that $\beta$ controls the convergence while $L_x$ changes the point of convergence. This is an expected behavior since changing $L_x$ will imply a change in $\lambda$ so a change in the objective function we are minimizing. We can also notice on this figure that when Equation~(\ref{eq:convergence-condition2}) is not fulfilled, which means when the approximate Lagrangian from Equation~\ref{equ:approx_lagr} is not a majorizer of the augmented Lagrangian, the algorithm diverges. We on the other side observe that the algorithm converge even when Equation~(\ref{eq:convergence-condition2}) is not respected.
\begin{figure}[H]
    \centering
    \includegraphics[width=0.4\linewidth]{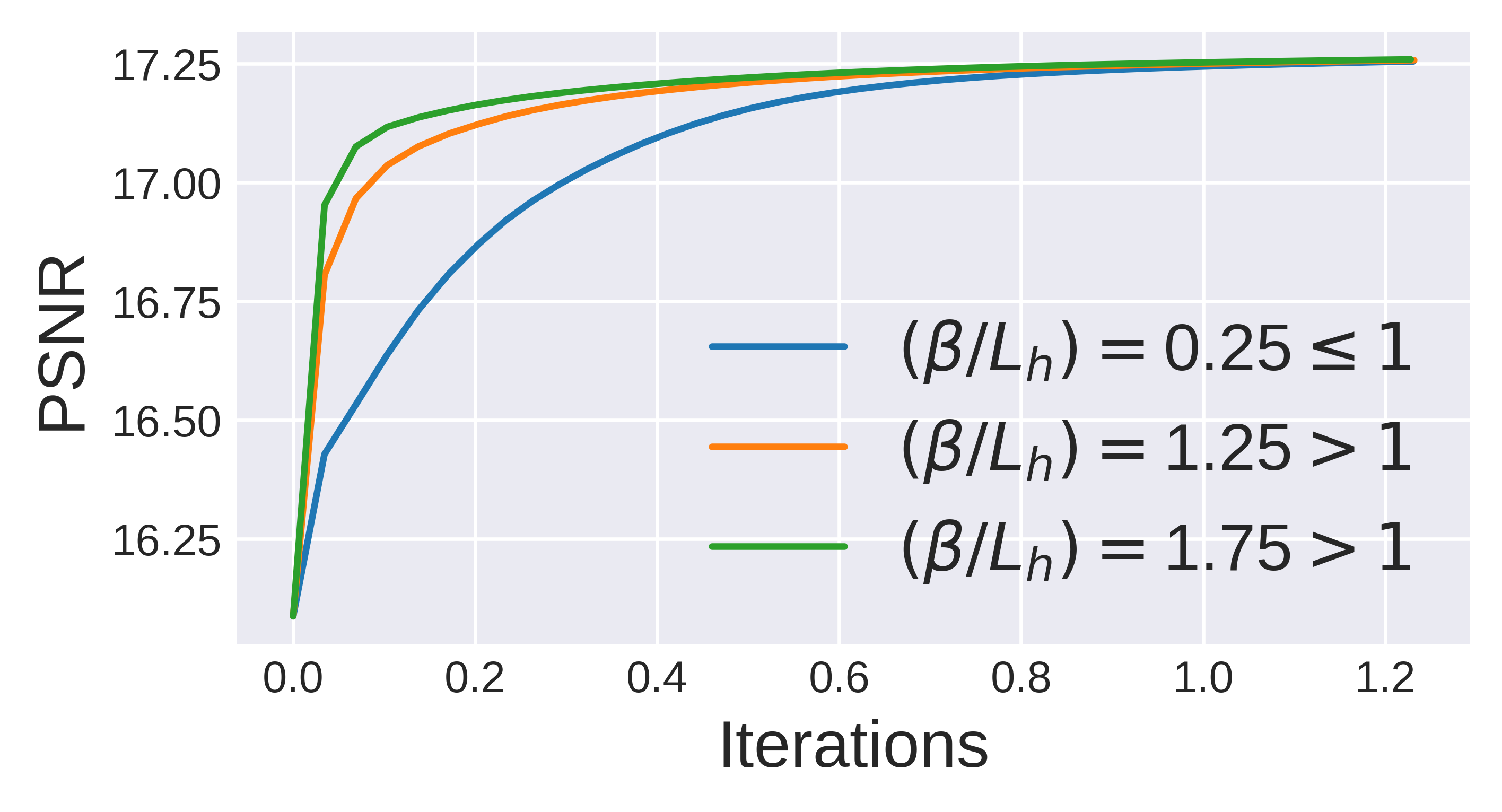} \hspace{1cm} \includegraphics[width=0.4\linewidth]{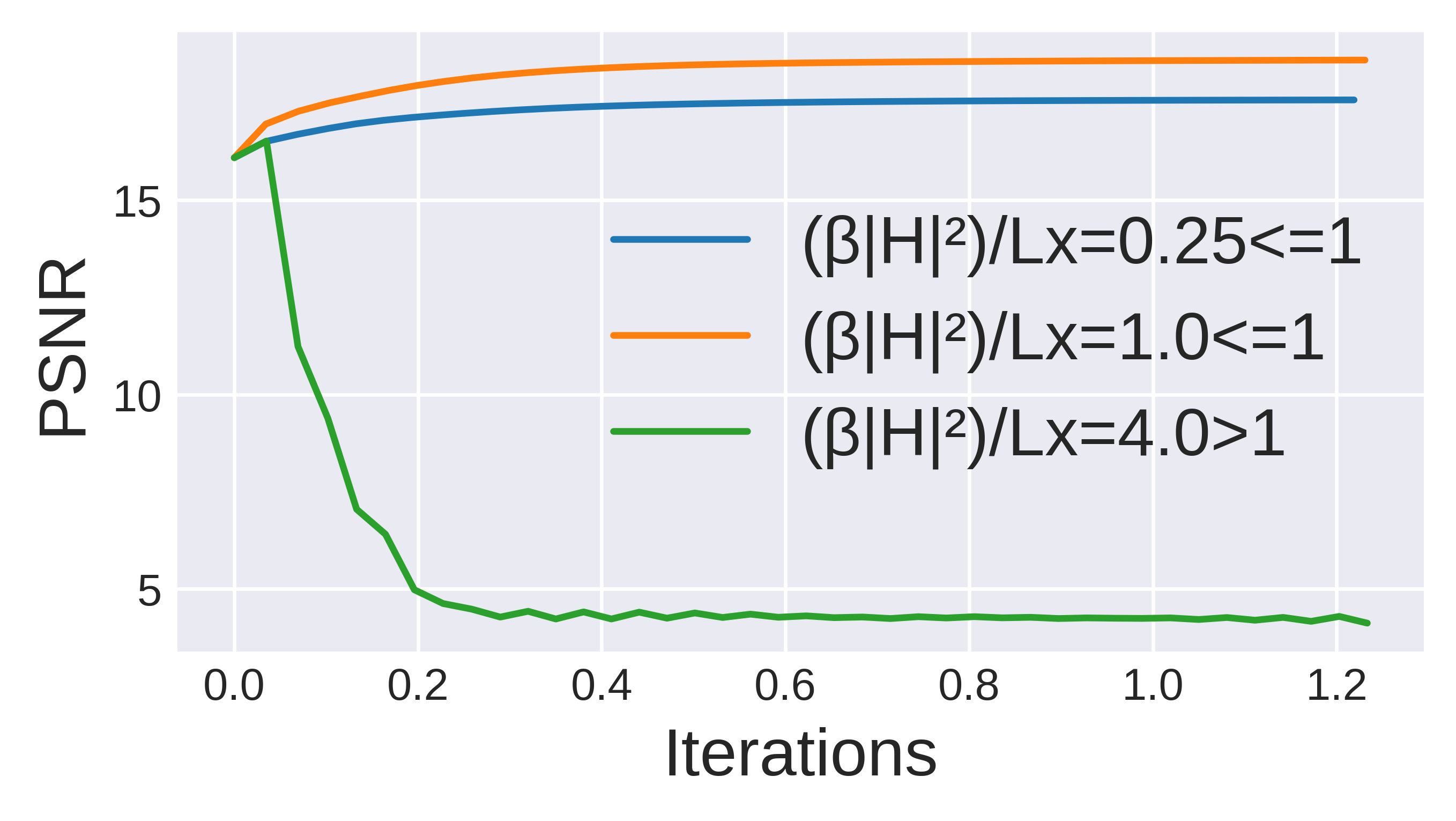}
    \caption{Parameters influence}
    \label{fig:params_influence}
\end{figure}

\subsection{Visual results}

We did not provide visual results in the main paper since our algorithm reaches similar performance as PnP-ADMM+CG. Thus there is not much to compare between the images. However, Figure~\ref{fig:visual_res} shows some visual results we obtained for the task of deblurring spatially-varying blur.
\begin{figure}[H]
    \centering
    \includegraphics[width=0.7\linewidth]{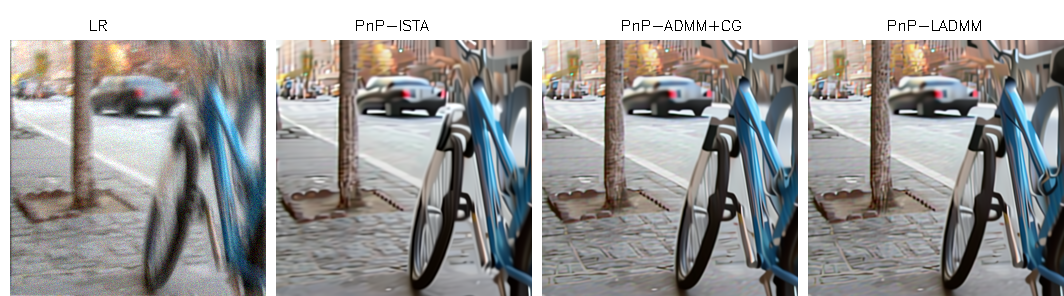} \includegraphics[width=0.7\linewidth]{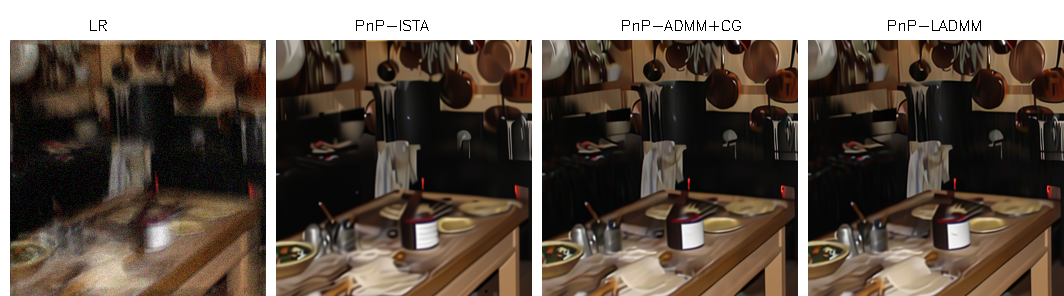}
    \caption{Visual results for deblurring spatially-varying blur.}
    \label{fig:visual_res}
\end{figure}

\subsection{Convergence of the splitting variables}

Theorem~\ref{theorm1} ensures the convergence of the splitting variables which means $\|x_{k+1}-x_k\|$ and $\|z_{k+1}-z_k\|$ converge. We provide in Figure~\ref{fig:conv-compare} a plot of $\|x_{k+1}-x_k\|$. We can see that PnP-ISTA has a fast convergence in the first iterations but then the convergence is very slow. On the other side, PnP-ADMM + CG convergence is slow but better than PnP-ISTA after few iterations. PnP-linearized-ADMM is the fastest algorithm. The advantage of plotting PSNR instead of $\|x_{k+1}-x_k\|$ is that it both shows the convergence of the algorithm and the quality of the fixed point.

\begin{figure}[H]
    \centering
    \includegraphics[width=0.5\linewidth]{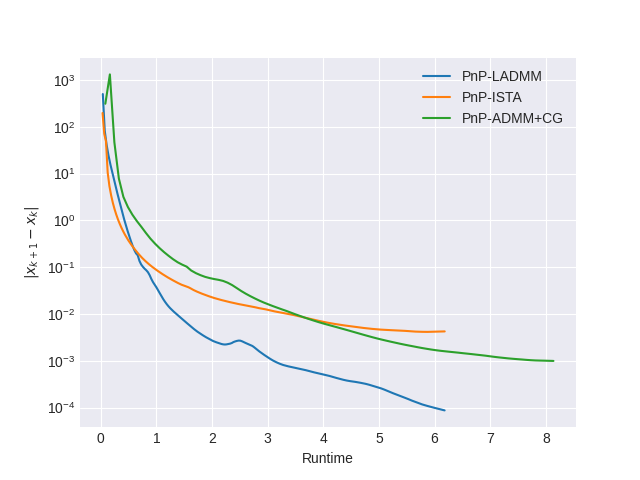}
    \caption{Convergence of the splitting variable $\|x_{k+1}-x_k\|$.}
    \label{fig:conv-compare}
\end{figure}

\end{proof}
}{}
\end{document}